\documentclass[letterpaper, 10 pt, conference]{ieeeconf}

\IEEEoverridecommandlockouts
\usepackage{tikz}
\usepackage{url}
\usepackage{amsmath}
\usepackage{amssymb}
\usepackage{mathtools}
\usepackage[compress]{cite}

\let\labelindent\relax

\usepackage{enumitem}

\allowdisplaybreaks

\DeclareMathOperator*{\argmin}{arg\,min}
\DeclareMathOperator*{\radius}{rad}
\DeclareMathOperator*{\trace}{Tr}

\newtheorem{theorem}{Theorem}

\newtheorem{proposition}{Proposition}
\newtheorem{definition}{Definition}
\newtheorem{lemma}{Lemma}
\newtheorem{assumption}{Assumption}
\newtheorem{remark}{Remark}
\newtheorem{problem}{Problem}

\newcommand{\revise}[1]{{\color{black}#1}}

\newcommand*{\LONGVERSION}{}

\title{Learning Safety Filters for Unknown Discrete-Time Linear Systems\thanks{An early version of this paper was presented at NeurIPS 2021 Workshop on Safe and Robust Control of Uncertain Systems, which does not have a proceedings and is only intended for dissemination of the results.}}

\author{Farhad Farokhi\thanks{F. Farokhi and M. Zamani are with the University of Melbourne (\{farhad.farokhi,mohammad.zamani\}@unimelb.edu.au). A. S. Leong is with the Defence Science and Technology (DST) Group, Australia (alex.leong@defence.gov.au). At the time of working on this paper, M. Zamani was also affiliated with the DST Group. I. Shames is with the CIICADA Lab, Australian National University (iman.shames@anu.edu.au).}, Alex S. Leong, Mohammad Zamani, and Iman Shames\thanks{The work of F. Farokhi is supported by a research contract (ID10298) from the Defence Science and Technology (DST) Group, Australia. The work of I.~Shames is partially supported by the Australian Government, via grant AUSMURIB000001 associated with ONR MURI N00014-19-1-2571.}
}

\begin{document}

\maketitle

\begin{abstract}
	A learning-based safety filter is developed for discrete-time linear time-invariant systems with unknown models subject to Gaussian noises with unknown covariance. Safety is characterized using polytopic constraints on the states and control inputs. The empirically learned model and process noise covariance with their confidence bounds are used to construct a robust optimization problem for minimally modifying nominal control actions to ensure safety with high probability. The optimization problem relies on tightening the original safety constraints. The magnitude of the tightening is larger at the beginning since there is little information to construct reliable models, but shrinks with time as more data becomes available. 
\end{abstract}

\section{Introduction}
It is often desired to ensure \textit{safety} of a controlled system. Safety can be defined as maintaining the system's states and control inputs inside a well-defined set, referred to as \textit{safety set}. 
For instance, robots must be maneuvered in complicated previously-unseen environments without collisions, and phases and voltages in power system must be maintained within pre-defined bands to avoid blackouts. Controllers often ensure safety using reliable models of the system and environment. Models are required to extrapolate the behaviour of the system given the current state and the designed input sequences. Models are however subject to unknown uncertainties or might even be entirely unknown. Irrespective of the accuracy of the model in laboratory conditions, unknown or varying environmental factors, such as slippage and wind, can render the model uncertain. When facing  uncertainties, we can consider their worst-case magnitude to ensure safety robustly. However, robust safety can result in conservative controllers. Alternatively, we can utilize real-time data to ``learn'' representations or models of the uncertainty. Thus we must ensure safety based on inaccurate time-varying models that fit the data on the fly. This is the topic of the current paper. 

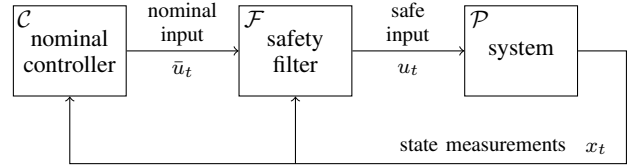
\begin{figure}[t]
	\centering
	\begin{tikzpicture}
		\node[rectangle,draw,minimum width=1.5cm,minimum height=1.2cm] (1) at (0,0) {};
		\node[]  at (0,0) {\small system};
		\node[rectangle,draw,minimum width=1.5cm,minimum height=1.2cm] (2) at (-3,0) {};
		\node[] at (-3,0) {\small 
			\begin{minipage}{1.4cm}
				\centering
				safety filter
			\end{minipage}};
		\node[rectangle,draw,minimum width=1.5cm,minimum height=1.2cm] (3) at (-6,0) {};
		\node[] at (-6,0) {\small 
			\begin{minipage}{1.7cm}
				\centering
				nominal controller
			\end{minipage}};
		\draw[->] (3) -- (2);
		\draw[->] (2) -- (1);
		\node[] at (-1.5,.35) {\footnotesize 
			\begin{minipage}{1.6cm}
				\centering
				safe \\ input
		\end{minipage}};
		\node[] at (-4.5,.35) {\footnotesize 
			\begin{minipage}{1.2cm}
				\centering
				nominal input
			\end{minipage}};
		\node[] at (-.5,-1.25) {\footnotesize 
			\begin{minipage}{2.5cm}
				\centering
				state measurements
		\end{minipage}};
		\node[] at (-0.55,.42) {\small $\mathcal{P}$};
		\node[] at (-3.55,.42) {\small $\mathcal{F}$};
		\node[] at (-6.6,.42) {\small $\mathcal{C}$};
		\node[] at (-4.5,-.25) {\footnotesize $\bar{u}_t$};
		\node[] at (-1.5,-.25) {\footnotesize $u_t$};
		\node[] at (1.,-1.3) {\footnotesize $x_t$};
		\draw[->] (1) -- (1.4,0) -- (1.4,-1.5) -- (-6,-1.5) -- (3);
		\draw[->] (-3,-1.5) -- (2);
	\end{tikzpicture}
	\caption{Safety filter $\mathcal{F}$ ensures satisfaction of state and control constraints for linear discrete-time system $\mathcal{P}$ with known model subject to zero-mean Gaussian uncertainty with unknown covariance. The safety filter minimally modifies the control input from nominal controller $\mathcal{C}$ via an optimization program to guarantee that the states in the next time step remain in the safe set with high probability. }
	\label{fig:safety_filter}
 \vspace{-8mm}
\end{figure}

In this paper, a learning-based \textit{safety filter} is developed for systems with unknown discrete-time linear time-invariant dynamics subject to a zero-mean Gaussian process noise with unknown covariance. The safety set is characterized using polytopic constraints on the states and control inputs. A block diagram of the closed-loop system using the safety filter is illustrated in Figure~\ref{fig:safety_filter}. The safety filter relies on two learning-based components: regression for learning the system model and empirical covariance of the process noise. The learned model and empirical process noise covariance, along with their confidence bounds, are used to construct a robust optimization problem for minimally modifying (in terms of a distance metric) nominal control actions to ensure safety with high probability. The nominal control can be generated by standard controllers, such as proportional-integral-derivative controllers, or by learning-based controllers, such as reinforcement learning. 
Finally, we propose directly optimizing the closed-loop performance by solving a model predictive control problem with tightened constraints instead of projecting nominal control inputs into the set of control signals to ensure safety. Similar to the projection-based safety filter, the magnitude of constraint tightening, which is dictated by the confidence in the learned models, is larger at the beginning since there is little information to construct reliable models, but shrinks with time as more data becomes available. \revise{
Note that individual elements of this paper, such as regularized least-squares learning of the model, empirical estimation of the covariance matrices of the process noise, and robust optimization for modifying control inputs, are traditionally investigated separately in the literature. The \textit{main} contributions of this paper are to combine these methods to develop a rigorous analysis of learning-based safety filters for unknown systems, and to develop computationally-efficient safety filters that are missing from the literature.}

Safe learning-based control, where the system model and uncertainties are unknown and must be estimated, has gained much attention recently~\cite{TaylorSingletaryYueAmes, ChengOroszMurrayBurdick,ChengKhojastehAmesBurdick,JagtapPappasZamani,ChoiCastanedaTomlinSreenath}. A popular approach is to learn models using Gaussian processes~\cite{ChengOroszMurrayBurdick, ChengKhojastehAmesBurdick,JagtapPappasZamani}. Also, there are many definitions for safety in reinforcement learning~\cite{garcia2015comprehensive}, but the approach of this paper relates more to reinforcement learning with constraints~\cite{marvi2021safe,fisac2018general,ChengOroszMurrayBurdick,fulton2018safe}. However, constraints in this paper are stage-wise as opposed to constraints on accumulated penalties over the planning horizon in constrained reinforcement learning. The above-mentioned studies share a common assumption that the learning of the models and uncertainties are done prior to control or that we can alternate between learning and control with batch learning~\cite{ChoiCastanedaTomlinSreenath, JagtapPappasZamani,ChengKhojastehAmesBurdick, ChengOroszMurrayBurdick,TaylorSingletaryYueAmes}. In contrast, our main interest is to perform learning and control simultaneously while new state measurements arrive, and to maintain safety based \revise{on} time-varying inaccurate models.

There are few studies that consider safety in simultaneous learning and control. The work of~\cite{DevonportYinArcak} uses confidence of learned Gaussian processes when modelling non-linear systems and environments to make decisions regarding safety based on the number of measurements used for learning. Although powerful, that work does not provide computationally efficient methods for ensuring safety, as their framework relies on Lyapunov functions, which can be difficult to find or compute for general systems. Another relevant study is~\cite{farokhi_safe_lerning}, which proposes computationally-efficient methods for projecting control signals into safe sets by computing the confidence of learning additive Gaussian models. However, that work only considers learning stochastic disturbances caused by the environment and assumes that the underlying model of the system is known. Most recently, a single-trajectory learning-based feedback scheme that ensures safety was proposed in~\cite{li2021safe}. In contrast to that study, the current paper does not focus on computing feedback functions but rather emphasizes constraint tightening for computing control actions using optimization problems. {The approach of this paper} results in a less complex for optimization problem; however, it requires sequentially solving optimization problems, which can be only done if only there is dedicated on-board computational capability. \revise{Similarly, learning-based optimal control over an infinite horizon was considered in~\cite{dean2019safely}. The emphasize in that paper was also on computing linear feedback policies using semi-definite programming. Safe learning for stochastic jump linear systems using semi-definite programming was considered in~\cite{schuurmans2019safe}. A myopic safety-constrained optimization was presented for water distribution networks in~\cite{9659138}. However, in that paper, constraint shrinking in response to learned model uncertainty was not considered.} Learning-based model predictive control with safety constraints have been proposed in~\cite{didier2021adaptive,lorenzen2019robust,wabersich2018linear}. These studies prove recursive feasibility and stability. However, computational issues, such as relying on polytopic sets for learning the model (with increasing numbers of polytopes or vertices with time), development of robust positively invariant sets, and requiring potentially high-dimensional parametric feedback functions, can stifle their implementation in practice. 

\revise{
The rest of the paper is organized as follows. First, the mathematical problem formulation is presented in Section~\ref{sec:safelearning}. Section~\ref{section:background} overviews confidence bounds for regularized least-squares learning of the system model. The data-driven safety filter is presented and analyzed in Section~\ref{sec:robustoptimization}. In Section~\ref{sec:learning_noise}, we reduce the conservatism of the safety filter by using the empirical covariance of the process noise in addition to learning the model parameters. Finally, Section~\ref{sec:conclusions} concludes the paper and presents directions for future research.

}

\subsubsection*{Notation}
Sets are denoted by calligraphic letters, such as $\mathcal{A}$. Matrices are denoted by capital Roman letters, such as $A$. The $i$-th row of $A$ is denoted by $A_i$. The entry in the $i$-th row and the $j$-th column of matrix $A$ is $a_{ij}$. Scalars and vectors are denoted by lowercase Roman and Greek letters, such as $x$ and $\theta$. Similarly, the $i$-th entry of vector $x$ is denoted by $x_i$. Let $\mathcal{S}_{++}^n$ and $\mathcal{S}_{+}^n$ refer to the sets of symmetric positive definite and positive semi-definite matrices in $\mathbb{R}^{n\times n}$. In what follows, $A \succ  B$ and $A \succeq B$, respectively, signify that $A-B\in \mathcal{S}_{++}^n$ and $A-B\in\mathcal{S}_{+}^n$. The smallest and the largest singular values of matrix $Y$ are, respectively, denoted by $\sigma_{\min}(Y)$ and $\sigma_{\max}(Y)$. Vector $e_i$ denotes the column-vector with all entries zero except the $i$-th entry, which is equal to one. For any $x\in\mathbb{R}^n$, $\|x\|$ denotes its Euclidean norm, i.e., $\|x\|=(\sum_{i=1}^n x_i^2)^{1/2}$. For any $A\in\mathbb{R}^{n\times m}$, $\|A\|$ denotes the induced matrix norm $\|A\|=\sup_{\|x\|=1}\|Ax\|$ and $\|A\|_F$ denotes the Frobenius norm $\|A\|_F=(\sum_{i=1}^n\sum_{j=1}^m a_{ij}^2)^{1/2}$. For any set $\mathcal{A}\subseteq\mathbb{R}^n$, $\radius(\mathcal{A})$ is its radius, i.e., $\radius(\mathcal{A})=\sup_{a\in\mathcal{A}} \|a\|$. For any signal $x[\cdot]$, $x[k_0\!:\!k_1]$ with $k_1\geq k_0$ denotes the sequence $(x[k_0],x[k_0+1],\dots,x[k_1])$. For $a,b\in\mathbb{R}^n$, $a\leq b$ signifies that the inequality holds entry-wise.

\section{Problem Formulation}\label{sec:safelearning}
Consider a linear time-invariant discrete-time system:
\begin{align} \label{eqn:system}
	x[k+1]=Ax[k]+Bu[k]+w[k],
\end{align}
where $x[k]\in\mathbb{R}^n$ is the state, $u[k]\in\mathbb{R}^m$ is the control input, and $w[k]\in\mathbb{R}^n$ is the process noise. The process noise is composed of a sequence of independently and identically distributed (i.i.d.) zero-mean Gaussian random variables with covariance $W\in\mathcal{S}_+^n$. \textit{Model parameters $A$, $B$, and $W$ are unknown and must be learned}. Safety is encoded by time-varying polytopic constraints:
\begin{align} \label{eqn:state_constraint}
	x[k]\in\mathcal{X}_k:=\{x\,|\, H[k]x\leq h[k]\}.
\end{align}
The control action is also constrained by 
\begin{align}
	u[k]\in\mathcal{U}:=\{u\,|\, Eu\leq f\}.
\end{align}
We make the following standing assumptions on covariance of the process noise, magnitude of the model parameters, and radii of the control and state constraint sets. 
\begin{assumption} \label{assum:all}
	There exists known constants:
	\begin{enumerate}[leftmargin=*,label={{\alph*:}}, ref={\theassumption.\alph*}]
		\item \label{assum:bound_on_variance}%
		$r>0$ such that $W\preceq rI$.
		\item \label{assum:bound_on_model}%
		$s>0$ such that 
		$\|[\begin{array}{cc}
			\!\!A & B\!\!
		\end{array}]\|_F\leq s$. 
		\item \label{assum:bound_on_radius}%
		$d>0$ such that $\radius(\mathcal{X}_k)+\radius(\mathcal{U})\leq d$.
	\end{enumerate}
\end{assumption}

When controlling a system with unknown model, the uncertainty of the learned model gets multiplied by the states and control inputs at the current time to determine the uncertainty of the state in the next time step; see~\eqref{eqn:def_vkk} and~\eqref{eqn:def_vk} below. Therefore, if the state and the control input are unbounded, the uncertainty of the state after making a decision can become large, which can complicate ensuring safety. Assumption~\ref{assum:bound_on_radius} ensures that the state and the control input are bounded so that we can avoid this problem. In practice, this assumption can be  relaxed. At the beginning when the uncertainty of the learned model is high, we can keep the states and the control actions restricted to small sets but, as our confidence in the learned model improves, we relax this assumption by gradually increasing the radii of the sets. Subsection~\ref{sec:limitations} presents another approach that partially relaxes Assumption~\ref{assum:bound_on_radius} and removes the need for requiring that the states remain within a bounded set with \textit{a priori} known radius for all times.

\begin{problem}
At $k\in\mathbb{N}\cup\{0\}$, given state measurements $x[0],\dots,x[k]$, find a procedure to \revise{compute a modified control input $u[k]\in\mathcal{U}$ based on a nominal control input $\bar{u}[k]$ by minimizing $\mathbf{d}(u[k],\bar{u}[k])$, where $\mathbf{d}(.,.)$ is a distance metric\footnote{An example of the distance metric is $\mathbf{d}(x,x')=\|x-x'\|$.},} \revise{subject to potentially tightened state and control constraints} to ensure the state in the next time step remains safe, i.e., $x[k+1]\in\mathcal{X}_{k+1}$, with high probability. 
\end{problem}

\section{Preliminary Results}
\label{section:background}

We use (regularized) least-squares to learn the model:
\begin{align} \label{eqn:least_squares}
	(\hat{A}[k],\hat{B}[k])\!\in\!\argmin_{(\bar{A},\bar{B})}\, &\Bigg[\sum_{t=0}^{k-1} \|x[t+1]\!-\!(\bar{A}x[t]\!+\!\bar{B}u[t])\|^2\nonumber\\
    &+\lambda(\|\bar{A}\|_F^2+\|\bar{B}\|_F^2)\Bigg],
\end{align}
where $\lambda>0$ is the regularization weight. Before we gather enough measurements, i.e., if $k<n(n+m)$, the least-squares problem~\eqref{eqn:least_squares} admits infinitely-many solutions without regularization, i.e., if $\lambda=0$. Regularization ensures that the least-squares problem~\eqref{eqn:least_squares} is strictly convex with a unique solution even in the absence of enough measurements. This also enables computing the confidence bounds for the learned model at all times. 

To analyze the safety filter, we need to better understand the moments of the random variable:
\begin{align}
    \revise{v[k',k]}:=(A-\hat{A}[k])x[k']+(B-&\hat{B}[k])u[k']\revise{,}\nonumber\\&\revise{\forall k'\geq k\geq 0.}\label{eqn:def_vkk}
\end{align}
\revise{Note that we can rewrite the system dynamics in~\eqref{eqn:system} as
\begin{align} \label{eqn:relation_x_and_v}
    x[k'+1]=\hat{A}[k]x[k']+\hat{B}[k]u[k']+\revise{v[k',k]}+w[k'].
\end{align}
Therefore, the random variable $\revise{v[k',k]}$ captures the error of forecasting the state at time $k'+1$, i.e., $x[k'+1]$, by  using the learned model based on the measurements up to time $k$, i.e., $(\hat{A}[k],\hat{B}[k])$.}
When $k'=k$, with slight abuse of notation, we write 
\begin{align}
    v[k]:= \revise{v[k,k]}
    =(A-\hat{A}[k])x[k]+(B-\hat{B}[k])u[k].\label{eqn:def_vk}
\end{align}

\begin{proposition} \label{cor:confidence} If $x[k']\in\mathcal{X}_{k'}$ and $u[k']\in\mathcal{U}$, then
	\begin{align*}
		\mathbb{P}\left\{\|\revise{v[k',k]}\|
		\leq \zeta n\beta_k({\delta}/{n})\right\}
		\geq 1-\delta,\quad \forall \revise{k'\geq k\geq 0},
	\end{align*}
	where $\zeta:=d/{\sqrt{\sigma_{\min}(V[k])}}$ and 
 \begin{align} \label{eqn:delta_def}
	\beta_k(\delta):= r\sqrt{2\log\left({\det(V[k])^{1/2}}/({\lambda^{n/2}\delta})\right)}+\lambda^{1/2}s
\end{align}
with $V[k]:=\lambda I+\hat{V}[k]$ and
\begin{align*}
	\hat{V}[k]\!:=\!\begin{bmatrix}
		x[0]^\top & u[0]^\top \\
		x[1]^\top & u[1]^\top \\
		\vdots & \vdots \\
		x[k\!-\!1]^\top & u[k\!-\!1]^\top 
	\end{bmatrix}^{\!\!\top} \!\!\!\begin{bmatrix}
		x[0]^\top & u[0]^\top \\
		x[1]^\top & u[1]^\top \\
		\vdots & \vdots \\
		x[k\!-\!1]^\top & u[k\!-\!1]^\top 
	\end{bmatrix}\!\!.
\end{align*}
\end{proposition}

\ifdefined\SHORTVERSION
\begin{proof}
See \cite[Appendix II]{farokhi_pre_print}.
\end{proof}
\fi 

\ifdefined\LONGVERSION
\begin{proof}
    See Appendix~\ref{proof:cor:confidence}.
\end{proof}
\fi 

Before presenting the following result, we need to define persistence of excitation, which is a common assumption in system identification and adaptive control~\cite{sastry2011adaptive}.

\begin{definition}[Persistence of Excitation] The system in \eqref{eqn:system} is persistently excited if there exists constants $\gamma\geq \alpha>0$ and an integer $T_0>0$ such that, $\forall k\in\mathbb{N}\cup\{0\}$,
	\begin{align*}
		\alpha I \preceq
		\begin{bmatrix} 
			\displaystyle\sum_{t=k}^{k+T_0-1} x[t]x[t]^\top &
			\displaystyle\sum_{t=k}^{k+T_0-1} x[t]u[t]^\top\\[.5em]
			\displaystyle\sum_{t=k}^{k+T_0-1} u[t]x[t]^\top &
			\displaystyle\sum_{t=k}^{k+T_0-1} u[t]u[t]^\top 
		\end{bmatrix}
		\preceq \gamma I.
	\end{align*}
\end{definition}

\begin{proposition} \label{cor:confidence_PoE} 
	If $x[k'],x[k'']\in\mathcal{X}_{k'}$, $u[k'],u[k'']\in\mathcal{U}$, and the persistence of excitation holds, then
	\begin{align*}
		\mathbb{P}\left\{\!\sqrt{\|\revise{v[k',k]}\|\;\|\revise{v[k'',k]}\|}
		\leq \zeta'_k n\beta_k({\delta}/{n})\!\right\}
		\geq &1-\delta,\\ &\forall \revise{k'\geq k\geq 0},
	\end{align*}
	where $\zeta'_k:=d/\sqrt{\lfloor k/T_0 \rfloor \alpha +\lambda}$.
\end{proposition}

\ifdefined\SHORTVERSION
\begin{proof}
See \cite[Appendix III]{farokhi_pre_print}.
\end{proof}
\fi 

\ifdefined\LONGVERSION
\begin{proof}
See Appendix~\ref{proof:cor:confidence_PoE}.
\end{proof}
\fi 

\begin{proposition} \label{prop:upperbound_expectation}
	Assume that $n\geq 2$. If $x[k'],x[k'']\in\mathcal{X}_{k'}$, $u[k'],u[k'']\in\mathcal{U}$, and the persistence of excitation holds, then
	\begin{align*}
		\mathbb{E}\{\|\revise{v[k',k]}\|^2\}&\leq L_2(k),\quad \revise{\forall k',k''\geq k\geq 0}\\
		\mathbb{E}\{\|\revise{v[k',k]}\|^2\|\revise{v[k'',k]}\|^2\}&\leq L_4(k),\quad \revise{\forall k',k''\geq k\geq 0}
	\end{align*}
	where
	\begin{align*}
		L_2(k)
		\!:=&\frac{\lambda s^2d^2n^2}{\lfloor k/T_0 \rfloor \alpha +\lambda}\\
            &+\frac{((\lfloor k/T_0 \rfloor + 1) \gamma +\lambda)}{\lambda}\frac{rd^2n^{\frac{7}{2}}(
			rn^{\frac{1}{2}}+
			\sqrt{\pi\lambda}s )}{(\lfloor k/T_0 \rfloor \alpha +\lambda)},\\
		L_4(k) 
        \!:=&\frac{\lambda^2 s^4d^4n^4}{(\lfloor k/T_0 \rfloor \alpha +\lambda)^2}\\
            &\!+\!\frac{((\lfloor k/T_0 \rfloor \!+\! 1) \gamma \!+\!\lambda)}{\lambda}\frac{8rd^4n^{\frac{11}{2}}(r^3n^{\frac{3}{2}}\!+\!
			\sqrt{\pi}\lambda^{\frac{3}{2}}s^3)}{(\lfloor k/T_0 \rfloor \alpha \!+\!\lambda)^2}.
	\end{align*}
	Evidently, $L_2(k)=\mathcal{O}(1)$ and $L_4(k)=\mathcal{O}(1/k)$. 
\end{proposition}

\ifdefined\SHORTVERSION
\begin{proof}
See \cite[Appendix IV]{farokhi_pre_print}.
\end{proof}
\fi 

\ifdefined\LONGVERSION
\begin{proof}
See Appendix~\ref{proof:prop:upperbound_expectation}.
\end{proof}
\fi 
With these preliminary results in hand, we are ready to investigate the effectiveness of the safety filter.

\section{Data-Driven Safety Filter}
\label{sec:robustoptimization}
In this paper, we modify a nominal control input  $\bar{u}[k]$ at each iteration to ensure safety. Projection of the control action $\bar{u}[k]$ to a safe set can be done by solving:
\revise{
\begin{subequations}
		\label{eqn:optim_safe_tightened}
		\begin{align}
			u[k]\in&\argmin_{u\in\mathcal{U}}  \mathbf{d}(u,\bar{u}[k]),\\
			&\quad\, \mathrm{s.t.}\quad H[k+1](\hat{A}[k]x[k] + \hat{B}[k] u)\nonumber\\
            &\hspace{1.2in}\leq h[k+1]-\bar{e}[k+1],
		\end{align}
	\end{subequations}
	where}
	\begin{align} \revise{\label{eqn:definition_bar_e_Prop_4}
		\bar{e}_i[k+1]=\left(\!\frac{\displaystyle dn\beta_k\!\left(\frac{\delta}{2n}\right)}{\sqrt{\sigma_{\min}(V[k])}} \!+\!\sqrt{\frac{2rn}{\delta}}\right)\|H_i[k\!+\!1]^\top\!\|,}
	\end{align}
 and $\delta\in(0,1)$ is a design parameter determining the probability of violating the safety constraints, $w\in\mathbb{R}^n$ is an uncertainty term linked with the process noise, and $v\in\mathbb{R}^n$ is an uncertainty term linked with the (in)accuracy of the learned model. 

\revise{
\begin{theorem}
	\label{tho:safe_projection}
	Assume that problem~\eqref{eqn:optim_safe_tightened} is feasible. Then, by implementing the control action $u[k]$ extracted from~\eqref{eqn:optim_safe_tightened}, $\mathbb{P}\{ x[k+1]\in\mathcal{X}_{k+1}\}\geq 1-\delta$.
\end{theorem}
}

\begin{proof} See Appendix~\ref{proof:tho:safe_projection}.
\end{proof}

The constraint-tightening term in~\eqref{eqn:optim_safe_tightened} is composed of two independent terms: one is caused by the uncertainty of the learned model and the other stems from the process noise. We can show that the constraint-tightening term due to the uncertainty of the learned model goes to zero under persistence of excitation.

\subsection{Persistence of Excitation for Safety Filter}
Persistence of excitation is a common assumption in system identification and adaptive control, which ensures that the error of learning the model converges to zero almost surely as more samples are gathered. This is done by exciting the system along all directions.

\begin{proposition} \label{prop:PoE_model}
	Assume that $\|H_i[k+1]^\top\|$ is uniformly bounded and system \eqref{eqn:system} is persistently excited. Then, 
	$$
	\lim_{k\rightarrow \infty} ({dn}/{\sqrt{\sigma_{\min}(V[k])}}) \beta_k\left({\delta}/({2n})\right)\|H_i[k+1]^\top\|=0.
	$$
\end{proposition}

\ifdefined\SHORTVERSION
\begin{proof}
See \cite[Appendix VI]{farokhi_pre_print}.
\end{proof}
\fi 

\ifdefined\LONGVERSION
\begin{proof}
See Appendix~\ref{proof:prop:PoE_model}. 
\end{proof}
\fi

Proposition~\ref{prop:PoE_model} shows that, assuming persistence of excitation, the effect of the uncertainty caused by learning the model in the constraint tightening of\revise{~\eqref{eqn:definition_bar_e_Prop_4}} tends to zero as more measurements are gathered. Therefore, in the large $k$ regime, we can solve~\eqref{eqn:optim_safe_tightened} with 
$\bar{e}_i[k+1]=\sqrt{{2rn}/{\delta}}\|H_i[k+1]^\top\|.$
The remaining constraint tightening term in this optimization problem is caused by the process noise. Note that, because we have not attempted at learning the statistics of the process noise, we consider the worst-case scenario in light of Assumption~\ref{assum:bound_on_variance}. After recovering the model parameters, the techniques of \cite{farokhi_safe_lerning} can be used to learn the statistics of the noise and also shrink this term. This is formalized in Section~\ref{sec:learning_noise}. 

\subsection{Conservatism in Constraint Tightening}
\label{sec:limitations}
In~\revise{\eqref{eqn:optim_safe_tightened}}, the worst-case magnitude of the uncertainty term $v$, linked to the inaccuracy of the learned model, scales quadratically with $d$, which is an upper bound on the radii of $\mathcal{X}_k$ and $\mathcal{U}$. This is because the model uncertainty gets multiplied by the state and the control input, and can result in conservative behaviour when $\mathcal{X}_k$ and $\mathcal{U}$ are large sets. Furthermore, according to Assumption~\ref{assum:bound_on_radius}, we need to assume existence of a bounded set to which $x[k]$ belongs for all $k$. These factors can combine to increase the conservatism of the projection-based approach. By examining the steps of the proof \revise{of Proposition~\ref{cor:confidence}, which is used to prove Theorem~\ref{tho:safe_projection}}, we can show that 
$
\mathbb{P}\{\|v[k]\|^2
\leq {n^2(\|x[k]\|^2+\radius(\mathcal{U})^2)}\beta_k^2({\delta}/({2n}))/{\sigma_{\min}(V[k])} \}
\geq 1-{\delta}/{2}.
$
Therefore, we can relax~\revise{\eqref{eqn:optim_safe_tightened}} to
\begin{subequations} \label{eqn:optim_safe_1}
	\begin{align} 
		u[k]\!\in&\argmin_{u\in\mathcal{U}}  \mathbf{d}(u,\bar{u}[k]),\\
		&\quad\, \mathrm{s.t.}  \quad \revise{H[k+1](\hat{A}[k]x[k] + \hat{B}[k] u )}\nonumber\\
        &\hspace{1.2in}\revise{\leq \!h[k+1]-\hat{e}[k+1],}
	\end{align}
\end{subequations}
\revise{where}
\begin{align} 
    \revise{
		\hat{e}_i[k+1]=\Bigg(} &\revise{\frac{\displaystyle n\sqrt{\|x[k]\|^2+\radius(\mathcal{U})^2}}{\sqrt{\sigma_{\min}(V[k])}}\beta_k\left(\frac{\delta}{2n}\right)} \nonumber\\
        &\revise{+\sqrt{\frac{2rn}{\delta}}\Bigg)\|H_i[k+1]^\top\|.}
	\end{align}
\revise{Similarly, it can be proved} that, by implementing the control action $u[k]$ extracted from the optimization problem\revise{~\eqref{eqn:optim_safe_1}, if feasible}, $x[k+1]$ is safe with probability of at least $1-\delta$.
This clearly yields an improved performance because $\|x[k]\|^2+\radius(\mathcal{U})^2\leq (\|x[k]\|+\radius(\mathcal{U}))^2 \leq  d^2$ for all $k$ due to Assumption~\ref{assum:bound_on_radius}. Furthermore, we do not need to assume \textit{a priori} knowledge of $\sup_{k\geq 0}\|x[k]\|$. 

\subsection{Combining Controller and Safety Filter}
Instead of projecting nominal control inputs into the set of control signals that ensure the safety of the system, we can directly optimize the closed-loop performance by solving:
\begin{subequations} \label{eqn:optim_safe_MPC}
	\begin{align}
        \argmin_{\scriptsize
			\begin{array}{c}
				\bar{u}[k\!:\!k\!+\!T\!-\!1] \\
				\bar{x}[k\!+\!1\!:\!k\!+\!T]
			\end{array}
		}  &\!\!\!\!\sum_{t=k}^{k+T-1}\!\!\!\bar{u}[t]^\top R_t \bar{u}[t]\!+\!\!\!\!\sum_{t=k+1}^{k+T} \!\!\!(\bar{x}[t]^\top Q_t\bar{x}[t]\!+\!q_t^\top \bar{x}[t]),\\
		\mathrm{s.t.}\hspace{7mm}& \bar{u}[k\!:\!k\!+\!T\!-\!1]\!\in\!\mathcal{U}^T \\
		& \bar{x}[t+1]=\hat{A}[t]\bar{x}[t] + \hat{B}[t] \bar{u}[t],\nonumber\\
        &\hspace{.5in}\forall t\in\{k,\dots,k+T-1\},\\
		&\bar{x}[k]=x[k],\\
		&H[k\!+\!1]\bar{x}[k\!+\!1]\!\leq\! h[k\!+\!1]\!-\!\bar{e}[k\!+\!1],\\
		&\bar{x}[t]\in\mathcal{X}_t,  \forall t\in\{k+2,\dots,k+T\},
	\end{align}
\end{subequations}
where $T\in\mathbb{N}$ denotes the decision making horizon, \revise{$\mathcal{U}^T$ denotes the $T$-fold Cartesian product of the set $\mathcal{U}$}, $\bar{e}[k+1]$ is defined in\revise{~\eqref{eqn:definition_bar_e_Prop_4}}, and $R_t\in\mathcal{S}_{++}^m$, $Q_t\in\mathcal{S}_{+}^n$, and $q_t\in\mathbb{R}^n$ are the parameters of the cost function. This optimization problem is similar to the one solved in model predictive control~\cite{rawlings2009model}, with the exception that the safety constraints on the state for the next time step, i.e.,  $x[k+1]$, is tightened to ensure safety despite modelling uncertainty and process noise. Note that other safety constraints can be tightened following a similar line of reasoning; however, the conservatism increases for them as new measurements are not available or taken into consideration for shrinking the magnitude of the constraint tightening. Assuming that problem~\eqref{eqn:optim_safe_MPC} is feasible and, by implementing the control action $u[k]$ from the solution $u[k\!:\!k\!+\!H\!-\!1]$ of~\eqref{eqn:optim_safe_MPC}, $x[k+1]$ is safe with probability of at least $1-\delta$.

One positive aspect of the model predictive control formulation, as opposed to instantaneous or myopic projection of nominal control actions to ensure safety, is that the optimization problem is more likely to remain feasible. For instance, in obstacle avoidance, model predictive control looks ahead to avoid future states that can cause infeasibility down the track. However, this comes at the cost of an increased computational burden because of the longer horizon and increased dimension. An important direction for future research is to establish recursive feasibility of the proposed learning-based model predictive control, i.e., establishing conditions under which, if~\eqref{eqn:optim_safe_MPC} is feasible at time $k$, it is also feasible at time $k+1$. To be able to establish recursive feasibility, we need to prove that the uncertainty sets for the model matrices and the covariance matrix are recursively contained, i.e., access to more measurements does not increase uncertainty in some directions. Furthermore, we must search over the set of feedback policies rather than control inputs. Given these properties in addition to a robust positively invariant safe set, we can use standard recursive feasibility arguments from robust model predictive control. These requirements however can limit the computationally-friendly nature of the constraint-tightening projection-based approach in this paper. 

\section{Learning of Process Noise Covariance}\label{sec:learning_noise}
In this section, the covariance of the process noise is estimated empirically to reduce the conservatism of working with only the upper bound  in Assumption~\ref{assum:bound_on_variance}. In particular, we use the empirical covariance of the process noise:
\begin{align*}
	\revise{\widehat{W}[k,k_0]}=\frac{1}{k-k_0}\sum_{t=k_0+1}^k \hat{w}[t|k_0]\hat{w}[t|k_0]^\top,
\end{align*}
where $\revise{\hat{w}[k,k_0]}:=x[k+1]-(\hat{A}[k_0]x[k]+\hat{B}[k_0]u[k])$. For all $k>k_0\geq 0$, we ensure safety by projecting the control action $\bar{u}[k]$ using
\revise{
\begin{subequations} \label{eqn:optim_safe_tightened_with_W}
		\begin{align}
			u[k]\in &\argmin_{u\in\mathcal{U}}  \mathbf{d}(u,\bar{u}[k]),\\
			&\quad\, \mathrm{s.t.}  \quad H[k+1](\hat{A}[k_0]x[k] + \hat{B}[k_0] u)\nonumber\\
                &\hspace{1.2in}\leq h[k+1]-\tilde{e}[k+1],
		\end{align}
	\end{subequations}}
where 
\revise{
\begin{align}
		\tilde{e}_i[k+1]=&\frac{dn}{\sqrt{\sigma_{\min}(V[k])}} \beta_k\left(\frac{\delta}{3n}\right)\|H_i[k+1]^\top\|\nonumber\\
        &+\sqrt{\frac{3n}{\delta}}\left\|\Pi_{k,k_0}^{1/2}H_i[k+1]^\top\right\|,
	\end{align}
}and
\begin{align*}
	\Pi_{k,k_0}^{-1}:=&\revise{\widehat{W}[k-1,k_0]}\\
 &+\sqrt{\frac{3}{\delta}\left(\!2L_4(k_0)^2\!+\!\frac{8rL_2(k_0)}{k-k_0}\!+\!\frac{2r^2n(n\!+\!1)}{k-k_0}\right)} I.
\end{align*}

\begin{theorem} \label{tho:safe_with_W}
	Assume that problem~\eqref{eqn:optim_safe_tightened_with_W} is feasible and $n\geq 2$. Then, by implementing the control action $u[k]$ extracted from~\eqref{eqn:optim_safe_tightened_with_W}, $\mathbb{P}\{ x[k+1]\in\mathcal{X}_{k+1}\}\geq 1-\delta$.
\end{theorem}

\ifdefined\SHORTVERSION
\begin{proof}
See \cite[Appendix VII]{farokhi_pre_print}.
\end{proof}
\fi 

\ifdefined\LONGVERSION
\begin{proof}
    See Appendix~\ref{proof:tho:safe_with_W}.
\end{proof}
\fi

\revise{
\begin{remark}
The need for the assumption $n \geq  2$ in Theorem~\ref{tho:safe_with_W} arises from an inequality (i.e., $\delta^{n/2}\leq \delta$ for $\delta\in(0,1)$ and $n\geq 2$) used to prove Proposition~\ref{prop:upperbound_expectation}. Although this assumption seems to be technical, we have not been able to relax it. 
\end{remark}

\begin{remark}
    By increasing $k_0$, $L_4(k_0)$ decreases, which can potentially reduce the constraint tightening term. This is because, by increasing $k_0$, the accuracy of the learned model improves. However, by increasing $k_0$, $k-k_0$ gets smaller, which can potentially increase the constraint tightening. This trade-off stems from the fact that only $k_0$ measurements are used to learn the model parameters $(\hat{A}[k_0],\hat{B}[k_0])$ (so by increasing $k_0$ the learned model becomes more accurate) while the remaining $k-k_0$ measurements are used to empirically estimate the covariance of the process noise (so by increasing $k_0$ the empirical covariance becomes less reliable). This fundamental trade-off cannot be avoided unless the entire set of measurements are used to simultaneously learn the model parameters and estimate the covariance of the process noise. However, this approach complicates the proofs significantly and worsens the tightness of the bounds by generating extra cross-correlation terms. This is a trade-off that must be considered when choosing $k_0$. 
\end{remark}
}

\section{Conclusions} \label{sec:conclusions}
We considered safe learning-based control for discrete-time linear time-invariant  dynamical systems when the system model and the process noise covariance are unknown but bounded. We used regularized least-squares estimation to learn the model online and used the empirical covariance of the noise. We relied on the confidence bounds of the learned system model and the empirical process noise covariance to modify the control inputs via a robust optimization problem with time-varying safety constraints. We reformulated the problem in a computationally-friendly optimization problem for ensuring safety based on constraint tightening. Future work can focus on noisy output measurements and learning nonlinear systems using Gaussian processes.

\bibliography{ref}
\bibliographystyle{ieeetr}

\appendices

\ifdefined\SHORTVERSION
\section{Useful Lemma}
\begin{lemma} \label{lemma:robust} 
	For $W\succeq 0$ and $d\geq 0$, 
	$
	\{u\,|\,a^\top u+b^\top w \leq c,\forall w: w^\top W w\leq d \} =\{u\,|\,a^\top u\leq c-\sqrt{d} \|W^{-1/2}b\|\}
	$.
\end{lemma}

\begin{proof}
	With the change of variables $\bar{w}=W^{1/2}w$ and $\bar{b}=W^{-1/2} b$, we have 
	$
	\{u\,|\,a^\top u+b^\top w \leq c,\forall w: w^\top W w\leq d \}
	=
	\{u\,|\,a^\top u+\bar{b}^\top\bar{w} \leq c,\forall \bar{w}: \bar{w}^\top \bar{w}\leq d \}.
	$
	Then, following the approach of~\cite[Example~1.3.3]{ben2009robust}, we can obtain $\{u\,|\,a^\top u+\bar{b}^\top\bar{w} \leq c,\forall \bar{w}: \bar{w}^\top \bar{w}\leq d \}=\{u\,|\,\sqrt{d}\|\bar{b}\| \leq c-a^\top u \}$. 
\end{proof}
\fi 

\ifdefined\LONGVERSION

\section{Useful Lemmas}
We build on the results of~\cite{abbasi2011improved}. Define
\begin{align*}
	\theta_i :=
	\begin{bmatrix}
		A_i^\top \\ B_i^\top 
	\end{bmatrix},
\end{align*}
where $A_i$ and $B_i$ are, respectively, the $i$-th rows of matrices $A$ and $B$. Similarly, let $x_i$ and $w_i$ denote the $i$-th entries of vectors $x$ and $w$. We can rearrange the system dynamics in~\eqref{eqn:system} to obtain
\begin{align*}
	X_i[k]
	=
	Z[k]
	\theta_i+
	W_i[k],
\end{align*}
where
\begin{align*}
	X_i[k]=\begin{bmatrix}
		x_i[1] \\
		x_i[2] \\
		\vdots \\
		x_i[k]
	\end{bmatrix},\qquad 
	W_i[k]=\begin{bmatrix}
		w_i[0] \\
		w_i[1] \\
		\vdots \\
		w_i[k-1]
	\end{bmatrix},
\end{align*}
and
\begin{align*}
    Z[k]=\begin{bmatrix}
		x[0]^\top & u[0]^\top \\
		x[1]^\top & u[1]^\top \\
		\vdots & \vdots \\
		x[k-1]^\top & u[k-1]^\top 
	\end{bmatrix}.
\end{align*}
The regularized least-squares estimate of $\theta_i$ is given by
\begin{align} \label{eqn:least_squares_smaller}
	\hat{\theta}_i[k] \in \argmin_{\bar{\theta}_i\in \mathbb{R}^{n+m}} \left[\|X_i[k]- Z[k]\bar{\theta}_i\|^2+\lambda \|\bar{\theta}_i\|^2\right].
\end{align}
The solution to this regularized least-squares problem is given by
\begin{align}
	\hat{\theta}_i[k]:=(Z[k]^\top Z[k]+\lambda I)^{-1}Z[k]^\top X_i[k].
\end{align}
The regularized estimates can be concatenated to get the learned model:
\begin{align}
	\begin{bmatrix}
		\hat{A}[k] & \hat{B}[k]
	\end{bmatrix}
	:=
	\begin{bmatrix}
		\hat{\theta}_1[k]^\top 
		\\
		\vdots \\
		\hat{\theta}_n[k]^\top 
	\end{bmatrix}.
\end{align}

\begin{lemma} \label{tho:confidence} Let $V[k]=Z[k]^\top Z[k]+\lambda I$. Then,
$\mathbb{P} \{\|V[k]^{1/2}(\hat{\theta}_i[k]-\theta_i)\|
\leq \beta_k(\delta/n),\forall i\}\geq 1-\delta,$
where $\beta_k(\delta)$ is defined in~\eqref{eqn:delta_def}.
\end{lemma}

\begin{proof}
    First, note that $\mathbb{E}\{\exp(\mu w_i[k])\}
	=\exp(\mu^2 \mathbb{E}\{w_i[k]^2\}/2)\leq\exp(\mu^2 r/2),$ 	where the equality follows from that $w_i[k]$ is a zero mean Gaussian random variable and the inequality follows from Assumption~\ref{assum:bound_on_variance}.
	Further, Assumption~\ref{assum:bound_on_model} implies that $\|\theta_i\|\leq \|[\!\!\begin{array}{cc} A & B\end{array}\!\!]\|_F\leq s$.
	Now, using~\cite[Theorem 2]{abbasi2011improved}, we get
	\begin{align*}
		\mathbb{P}
	\Bigg\{\|V[k]^{1/2}&(\hat{\theta}_i[k]-\theta_i)\|
		\leq \beta_k(\delta/n)\Bigg\}\geq 1-\frac{\delta}{n}.
	\end{align*}
	Then
	\begin{align*}
		\mathbb{P}
	\Bigg\{\|&V[k]^{1/2}(\hat{\theta}_i[k]-\theta_i)\|
		\leq \beta_k(\delta/n),\forall i\Bigg\} \\
		=&\mathbb{P}
		\Bigg\{\bigwedge_{i=1}^n \|V[k]^{1/2}(\hat{\theta}_i[k]-\theta_i)\|
		\leq \beta_k(\delta/n)\Bigg\}\\
		\\
		=&1-\mathbb{P}
		\Bigg\{\bigvee_{i=1}^n \|V[k]^{1/2}(\hat{\theta}_i[k]-\theta_i)\|
		> \beta_k(\delta/n)\Bigg\}\\
		\geq & 1-\sum_{i=1}^n\mathbb{P}
		\Bigg\{\|V[k]^{1/2}(\hat{\theta}_i[k]-\theta_i)\|
		> \beta_k(\delta/n)\Bigg\}\\
		=&1-\delta.
	\end{align*}
	This concludes the proof.
\end{proof}

The following lemmas are used in proving the results in this paper. 

\label{appendix:useful}
\begin{lemma} \label{lemma:sqrt} For $x_i>0, \forall i$, $\sqrt{x_1+\cdots +x_n}\leq \sqrt{x_1} +\cdots + \sqrt{x_n}.$
\end{lemma}

\begin{proof} For $x_i>0$, $i\in\{1,\dots,n\}$, we have $\sum_{i=1}^n x_i \leq \sum_{i=1}^n x_i+2\sum_{i=1}^n \sum_{j=1,\\ i\neq j}^n \sqrt{x_i x_j}=(\sqrt{x_1}+\dots+\sqrt{x_n})^2$.\end{proof}

\begin{lemma} \label{lemma:Forb_2norm} Let 
	$X=\begin{bmatrix}
			X_1^\top  & \dots & X_n^\top 
		\end{bmatrix}^\top.$
	Then, for any positive semi-definite matrix $Y$, 
	$\|XY\|_F\leq \sum_{i=1}^n \|Y(X_i)^\top\|.$
\end{lemma}

\begin{proof}
	The definition of Frobenius norm results in $\|X\|_F^2=\sum_{i=1}^n \|X_i^\top\|^2$. Furthermore, 
	\begin{align*}
		XY=
		\begin{bmatrix}
			X_1 Y\\
			\vdots \\
			X_n Y
		\end{bmatrix}.
	\end{align*}
	Therefore, 
	\begin{align*}
		\|XY\|_F^2
		&=\sum_{i=1}^n \|(X_iY)^\top\|^2=\sum_{i=1}^n \|Y(X_i)^\top\|^2.
	\end{align*}
	Finally, using Lemma~\ref{lemma:sqrt}, we get
	\begin{align*}
		\|XY\|_F
		=\sqrt{\sum_{i=1}^n \|Y(X_i)^\top\|^2}&\leq\sum_{i=1}^n \sqrt{\|Y(X_i)^\top\|^2} \\&= \sum_{i=1}^n \|Y(X_i)^\top\|.
	\end{align*}
	This concludes the proof.
\end{proof}

\begin{lemma} \label{lemma:robust} 
	For $W\succeq 0$ and $d\geq 0$, 
	$
	\{u\,|\,a^\top u+b^\top w \leq c,\forall w: w^\top W w\leq d \} =\{u\,|\,a^\top u\leq c-\sqrt{d} \|W^{-1/2}b\|\}
	$.
\end{lemma}

\begin{proof}
	With the change of variables $\bar{w}=W^{1/2}w$ and $\bar{b}=W^{-1/2} b$, we have 
	$
	\{u\,|\,a^\top u+b^\top w \leq c,\forall w: w^\top W w\leq d \}
	=
	\{u\,|\,a^\top u+\bar{b}^\top\bar{w} \leq c,\forall \bar{w}: \bar{w}^\top \bar{w}\leq d \}.
	$
	Then, following the approach of~\cite[Example~1.3.3]{ben2009robust}, we can obtain $\{u\,|\,a^\top u+\bar{b}^\top\bar{w} \leq c,\forall \bar{w}: \bar{w}^\top \bar{w}\leq d \}=\{u\,|\,\sqrt{d}\|\bar{b}\| \leq c-a^\top u \}$. 
\end{proof}

\section{Proof of Proposition~\ref{cor:confidence}}
\label{proof:cor:confidence}
Lemma~\ref{tho:confidence} implies that
\begin{align*}
	\mathbb{P}
\Bigg\{\|V[k]^{1/2}&(\hat{\theta}_i[k]-\theta_i)\| 
	\leq \beta_k(\delta/n),\forall i\Bigg\}\geq 1-\delta,
\end{align*}
and, as a result,
\begin{align*}
	\mathbb{P}
\Bigg\{&\sum_{i=1}^n\|V[k]^{1/2}(\hat{\theta}_i[k]-\theta_i)\| 
	\leq n\beta_k(\delta/n)\Bigg\}\\
 &\geq \mathbb{P}
\Bigg\{\|V[k]^{1/2}(\hat{\theta}_i[k]-\theta_i)\| 
	\leq \beta_k(\delta/n),\forall i\Bigg\}\\&\geq 1-\delta.
\end{align*}
Now note that
\begin{align*}
	\left\|\begin{bmatrix}
\hat{\theta}_1[k]^\top-\theta_1^\top \\
		\vdots \\
		\hat{\theta}_n[k]^\top -\theta_n^\top
\end{bmatrix}V[k]^{1/2}\right\| 
&\leq 
	\left\|\begin{bmatrix}
\hat{\theta}_1[k]^\top-\theta_1^\top \\
		\vdots \\
		\hat{\theta}_n[k]^\top -\theta_n^\top
	\end{bmatrix}V[k]^{1/2}\right\|_F\\
&\leq 
	\sum_{i=1}^n \|V[k]^{1/2}(\hat{\theta}_i[k] -\theta_i)\| ,
\end{align*}
where the first inequality follows from~\cite[Eq.~(3.241)]{gentle2007matrix} and the second inequality follows from Lemma~\ref{lemma:Forb_2norm}. Therefore, 
\begin{align*}
	\mathbb{P}
	&\left\{\left\|\begin{bmatrix}
		\hat{\theta}_1[k]^\top-\theta_1^\top \\
		\vdots \\
		\hat{\theta}_n[k]^\top -\theta_n^\top
	\end{bmatrix}V[k]^{1/2}\right\| \leq n\beta_k\left(\frac{\delta}{n}\right) \right\}\geq 1-\delta.
\end{align*}
Note that 
	\begin{align*}
	\revise{v[k',k]}=\begin{bmatrix}
		\theta_1^\top - \hat{\theta}_1[k]^\top\\
		\vdots \\
		\theta_n^\top - \hat{\theta}_n[k]^\top
	\end{bmatrix}
	\begin{bmatrix}
		x[k'] \\
		u[k']
	\end{bmatrix},
\end{align*}
and hence
\begin{align}
	\|\revise{v[k',k]}\|=&
	\left\|\begin{bmatrix}
		\hat{\theta}_1[k]^\top-\theta_1^\top \\
		\vdots \\
		\hat{\theta}_n[k]^\top -\theta_n^\top
	\end{bmatrix}
	\begin{bmatrix}
		x[k'] \\
		u[k']
	\end{bmatrix}    
	\right\| \nonumber\\
	=& \left\|\begin{bmatrix}
		\hat{\theta}_1[k]^\top-\theta_1^\top \\
		\vdots \\
		\hat{\theta}_n[k]^\top -\theta_n^\top
	\end{bmatrix}V[k]^{1/2}V[k]^{-1/2}
	\begin{bmatrix}
		x[k'] \\
		u[k']
	\end{bmatrix}    
	\right\| \nonumber\\
	\leq & \left\|\begin{bmatrix}
		\hat{\theta}_1[k]^\top-\theta_1^\top \\
		\vdots \\
		\hat{\theta}_n[k]^\top -\theta_n^\top
	\end{bmatrix}V[k]^{1/2}  
	\right\| \left\|V[k]^{-1/2}
	\begin{bmatrix}
		x[k'] \\
		u[k']
	\end{bmatrix}\right\| \label{eqn:middle_of_proof}\\
	\leq & \frac{1}{\sqrt{\sigma_{\min}(V[k])}}\left\|\begin{bmatrix}
		\hat{\theta}_1[k]^\top-\theta_1^\top \\
		\vdots \\
		\hat{\theta}_n[k]^\top -\theta_n^\top
	\end{bmatrix}V[k]^{1/2}  
	\right\| \nonumber\\
    &\hspace{.8in}\times\left\|
	\begin{bmatrix}
		x[k'] \\
		u[k']
	\end{bmatrix}\right\| \nonumber\\
	= & \frac{1}{\sqrt{\sigma_{\min}(V[k])}}\left\|\begin{bmatrix}
		\hat{\theta}_1[k]^\top-\theta_1^\top \\
		\vdots \\
		\hat{\theta}_n[k]^\top -\theta_n^\top
	\end{bmatrix}V[k]^{1/2}  
	\right\| \nonumber\\
    &\hspace{.8in}\times\sqrt{\|x[k']\|^2+\|u[k']\|^2}\label{eqn:middle_of_proof_the_other_one}\\
	\leq & \frac{1}{\sqrt{\sigma_{\min}(V[k])}}\left\|\begin{bmatrix}
		\hat{\theta}_1[k]^\top-\theta_1^\top \\
		\vdots \\
		\hat{\theta}_n[k]^\top -\theta_n^\top
	\end{bmatrix}V[k]^{1/2}  
	\right\| \nonumber\\
    &\hspace{.8in}\times\sqrt{\radius(\mathcal{X}_{k'})^2+\radius(\mathcal{U})^2}\nonumber\\
	\leq & \frac{1}{\sqrt{\sigma_{\min}(V[k])}}\left\|\begin{bmatrix}
		\hat{\theta}_1[k]^\top-\theta_1^\top \\
		\vdots \\
		\hat{\theta}_n[k]^\top -\theta_n^\top
	\end{bmatrix}V[k]^{1/2}  
	\right\| \nonumber\\
    &\hspace{.8in}\times(\radius(\mathcal{X}_{k'})+\radius(\mathcal{U})),\nonumber\\
	\leq & \frac{d}{\sqrt{\sigma_{\min}(V[k])}}\left\|\begin{bmatrix}
	\hat{\theta}_1[k]^\top-\theta_1^\top \\
	\vdots \\
	\hat{\theta}_n[k]^\top -\theta_n^\top
    \end{bmatrix}V[k]^{1/2}  
    \right\| ,\nonumber
\end{align}
and thus
\begin{align*}
	\mathbb{P}\Bigg\{&\|\revise{v[k',k]}\|
	\leq \zeta n\beta_k\left(\frac{\delta}{n}\right)\Bigg\}
	\\&\geq \mathbb{P}\left\{\left\|\begin{bmatrix}
		\hat{\theta}_1[k]^\top-\theta_1^\top \\
		\vdots \\
		\hat{\theta}_n[k]^\top -\theta_n^\top
	\end{bmatrix}V[k]^{1/2}\right\|\leq n\beta_k\left(\frac{\delta}{n}\right) \right\}
	\\&\geq 1-\delta.
\end{align*}
This concludes the proof.

\section{Proof of Proposition~\ref{cor:confidence_PoE}}
\label{proof:cor:confidence_PoE}
Following~\eqref{eqn:middle_of_proof}, we get
\begin{align}
	\|\revise{v[k',k]}\|
	\leq & \left\|\begin{bmatrix}
		\hat{\theta}_1[k]^\top-\theta_1^\top \\
		\vdots \\
		\hat{\theta}_n[k]^\top -\theta_n^\top
	\end{bmatrix}V[k]^{1/2}  
	\right\|\nonumber\\
    &\hspace{.8in}\times\left\|V[k]^{-1/2}
	\begin{bmatrix}
		x[k'] \\
		u[k']
	\end{bmatrix}\right\|\nonumber\\
	\leq & \frac{1}{\sqrt{\sigma_{\min}(V[k])}}\left\|\begin{bmatrix}
		\hat{\theta}_1[k]^\top-\theta_1^\top \\
		\vdots \\
		\hat{\theta}_n[k]^\top -\theta_n^\top
	\end{bmatrix}V[k]^{1/2}  
	\right\|\nonumber\\
    &\hspace{.8in}\times\left\|
	\begin{bmatrix}
		x[k'] \\
		u[k']
	\end{bmatrix}\right\|\nonumber\\
	\leq & \frac{1}{\sqrt{\sigma_{\min}(V[k])}}\left\|\begin{bmatrix}
		\hat{\theta}_1[k]^\top-\theta_1^\top \\
		\vdots \\
		\hat{\theta}_n[k]^\top -\theta_n^\top
	\end{bmatrix}V[k]^{1/2}  
	\right\|\nonumber\\
    &\hspace{.8in}\times(\radius(\mathcal{X}_{k'})+\radius(\mathcal{U})),\nonumber\\
	\leq & \frac{\radius(\mathcal{X}_{k'})+\radius(\mathcal{U})}{\sqrt{\lfloor k/T_0 \rfloor \alpha +\lambda}}\nonumber\\
    &\hspace{.8in}\times\left\|\begin{bmatrix}
		\hat{\theta}_1[k]^\top-\theta_1^\top \\
		\vdots \\
		\hat{\theta}_n[k]^\top -\theta_n^\top
	\end{bmatrix}V[k]^{1/2}  
	\right\|,\nonumber\\
	\leq & \frac{d}{\sqrt{\lfloor k/T_0 \rfloor \alpha +\lambda}}\left\|\begin{bmatrix}
		\hat{\theta}_1[k]^\top-\theta_1^\top \\
		\vdots \\
		\hat{\theta}_n[k]^\top -\theta_n^\top
	\end{bmatrix}V[k]^{1/2}  
	\right\|,\nonumber
\end{align}
where the second last inequality follows from~\eqref{eqn:matrix_inequality}. Hence,
\begin{align*}
    &\sqrt{\|\revise{v[k',k]}\|\;\|\revise{v[k'',k]}\|}\\
    &\quad \leq 
    \frac{d}{\sqrt{\lfloor k/T_0 \rfloor \alpha +\lambda}}\left\|\begin{bmatrix}
		\hat{\theta}_1[k]^\top-\theta_1^\top \\
		\vdots \\
		\hat{\theta}_n[k]^\top -\theta_n^\top
	\end{bmatrix}V[k]^{1/2}  
	\right\|,
\end{align*}
and as a result
\begin{align*}
	\mathbb{P}\Bigg\{&\sqrt{\|\revise{v[k',k]}\|\;\|\revise{v[k'',k]}\|}
	\leq \zeta'_k n\beta_k\left(\frac{\delta}{n}\right)\Bigg\}\\
	&\geq \mathbb{P}\left\{\left\|\begin{bmatrix}
		\hat{\theta}_1[k]^\top-\theta_1^\top \\
		\vdots \\
		\hat{\theta}_n[k]^\top -\theta_n^\top
	\end{bmatrix}V[k]^{1/2}\right\|\leq n\beta_k\left(\frac{\delta}{n}\right) \right\}\\
	&\geq 1-\delta.
\end{align*}
This concludes the proof.

\section{Proof of Proposition~\ref{prop:upperbound_expectation}}
\label{proof:prop:upperbound_expectation}

  Because of Proposition~\ref{cor:confidence_PoE} and the facts that $n^{n/2}\geq n$, $\delta^{n/2}\leq \delta$, and $\det(V[k]) \preceq ((\lfloor k/T_0 \rfloor + 1) \gamma +\lambda)^n$ from~\eqref{eqn:matrix_inequality}, we get~\eqref{eqn:super_long}.
    \begin{figure*}
    \begin{align}
    \mathbb{P}&\left\{\sqrt{\|\revise{v[k',k]}\|\|\revise{v[k'',k]}\|}
	> \frac{\displaystyle dn\left(r\sqrt{n\log\left(\frac{((\lfloor k/T_0 \rfloor + 1) \gamma +\lambda) n}{\lambda\delta}\right)}+\lambda^{1/2}s\right) }{\sqrt{\lfloor k/T_0 \rfloor \alpha +\lambda}} \right\}
	\nonumber\\
	&\hspace{.7in}= 
    \mathbb{P}\left\{\sqrt{\|\revise{v[k',k]}\|\|\revise{v[k'',k]}\|}
	> \frac{\displaystyle dn\left(r\sqrt{2\log\left(\frac{((\lfloor k/T_0 \rfloor + 1) \gamma +\lambda)^{n/2} n^{n/2}}{\lambda^{n/2}\delta^{n/2}}\right)}+\lambda^{1/2}s\right)}{\sqrt{\lfloor k/T_0 \rfloor \alpha +\lambda}} \right\}
	\nonumber\\
	&\hspace{.7in}\leq 
    \mathbb{P}\left\{\sqrt{\|\revise{v[k',k]}\|\|\revise{v[k'',k]}\|}
	> \frac{\displaystyle dn\left(r\sqrt{2\log\left(\frac{((\lfloor k/T_0 \rfloor + 1) \gamma +\lambda)^{n/2} n}{\lambda^{n/2}\delta}\right)}+\lambda^{1/2}s\right)}{\sqrt{\lfloor k/T_0 \rfloor \alpha +\lambda}} \right\}
    \nonumber\\
	&\hspace{.7in}\leq 
        \mathbb{P}\left\{\sqrt{\|\revise{v[k',k]}\|\|\revise{v[k'',k]}\|}
	> \frac{\displaystyle dn\left(r\sqrt{2\log\left(\frac{\det(V[k])^{1/2}n}{\lambda^{n/2}\delta}\right)}+\lambda^{1/2}s\right)}{\sqrt{\lfloor k/T_0 \rfloor \alpha +\lambda}} \right\}\leq \delta. \label{eqn:super_long}
    \end{align}
    \hrule 
    \end{figure*}
    By defining $u=dn(r\sqrt{2\log\left({\det(V[k])^{1/2}n/(\lambda^{n/2}\delta})\right)}+\lambda^{1/2}s)/{\sqrt{\lfloor k/T_0 \rfloor \alpha +\lambda}}$ and computing $\delta$ in terms of $u$, we get
    \begin{align}
    \mathbb{P}&\left\{\sqrt{\|\revise{v[k',k]}\|\|\revise{v[k'',k]}\|}
	> u\right\}\nonumber\\&\hspace{.5in}
	\leq \bar{\gamma}\exp\left(-\frac{1}{2}\left(\frac{u- \bar{\mu}}{\bar{\nu}}\right)^2\right),\quad \forall u\geq  \bar{\mu}, \label{eqn:exp_inequality}
    \end{align}
    where
    \begin{align*}
         \bar{\mu}:=&\frac{\sqrt{\lambda}sdn}{\sqrt{\lfloor k/T_0 \rfloor \alpha +\lambda}},\\ 
        \bar{\nu}:=&\frac{rdn^{3/2}}{\sqrt{2}\sqrt{\lfloor k/T_0 \rfloor \alpha +\lambda}},\\
        \bar{\gamma}:=&\frac{n((\lfloor k/T_0 \rfloor + 1) \gamma +\lambda)}{\lambda}.
    \end{align*}
	Using~\eqref{eqn:exp_inequality}, we get
	\begin{align*}
	\mathbb{E}\{\|&\revise{v[k',k]}\|^2\|\revise{v[k'',k]}\|^2\}\\
	&=
	\mathbb{E}\{(\sqrt{\|\revise{v[k',k]}\|\|\revise{v[k'',k]}\|})^4\}\\
	   &
	   =\int_0^{\infty} \mathbb{P}\{\sqrt{\|\revise{v[k',k]}\|\|\revise{v[k'',k]}\|}>u\}4u^3\mathrm{d}u
	   \\
	   &\leq \int_0^{ \bar{\mu}} 4u^3\mathrm{d}u
	   + 4\bar{\gamma}\int_{ \bar{\mu}}^{\infty} \exp\left(-\frac{1}{2}\left(\frac{u- \bar{\mu}}{\bar{\nu}}\right)^2\right) u^3\mathrm{d}u.
	\end{align*}
	Note that
	\begin{align*}
	    \int_{ \bar{\mu}}^{\infty} \exp&\left(-\frac{1}{2}\left(\frac{u- \bar{\mu}}{\bar{\nu}}\right)^2\right) u^3\mathrm{d}u\\
	    =&
	    \int_{0}^{\infty} \exp\left(-\frac{1}{2}\left(\frac{z}{\bar{\nu}}\right)^2\right) (z+ \bar{\mu})^3\mathrm{d}z\\
	    \leq &
	    \int_{0}^{\infty} \exp\left(-\frac{1}{2}\left(\frac{z}{\bar{\nu}}\right)^2\right) (4z^3+ 4\bar{\mu}^3)\mathrm{d}z\\
	    = &
	    4\int_{0}^{\infty} \exp\left(-\frac{z^2}{2\bar{\nu}^2}\right) z^3\mathrm{d}z\\
        &+ 4\int_{0}^{\infty} \exp\left(-\frac{z^2}{2\bar{\nu}^2}\right)  \bar{\mu}^3\mathrm{d}z\\
	    = &8\bar{\nu}^4+2\sqrt{2\pi}\bar{\nu}\bar{\mu}^3\\
	    = &2\bar{\nu}(4\bar{\nu}^3+\sqrt{2\pi}\bar{\mu}^3),
	\end{align*}
	where the first inequality follows from the convexity of $x\mapsto x^3$ over the positive real numbers. Therefore,
	\begin{align*}
	    \mathbb{E}\{\|\revise{v[k',k]}\|^2&\|\revise{v[k'',k]}\|^2\}\\
	   \leq&  \bar{\mu}^4
	   + 8\bar{\gamma}\bar{\nu}(4\bar{\nu}^3+\sqrt{2\pi}\bar{\mu}^3)\\
	   =&\frac{\lambda^2 s^4d^4n^4}{(\lfloor k/T_0 \rfloor \alpha +\lambda)^2}\\
        &+\frac{((\lfloor k/T_0 \rfloor + 1) \gamma +\lambda)}{\lambda}\\
        &\quad \times\frac{8rd^4n^{11/2}(r^3n^{3/2}+
	   \sqrt{\pi}\lambda^{3/2}s^3)}{(\lfloor k/T_0 \rfloor \alpha +\lambda)^2}.
	\end{align*}
	Also, by setting $k''=k'$, we get
	\begin{align*}
        \mathbb{P}\left\{\|\revise{v[k',k]}\|> u\right\}
	\leq \bar{\gamma}\exp\left(-\frac{1}{2}\left(\frac{u- \bar{\mu}}{\bar{\nu}}\right)^2\right), \forall u\geq  \bar{\mu}. 
    \end{align*}
	Similarly, we have
	\begin{align*}
	\mathbb{E}\{\|\revise{v[k',k]}\|^2\}
	  =&\int_0^{\infty} \mathbb{P}\{\|\revise{v[k',k]}\|>u\}2u\mathrm{d}u\\
	   \leq& \int_0^{ \bar{\mu}} 2u\mathrm{d}u\\
	   &+ 2\bar{\gamma}\int_{ \bar{\mu}}^{\infty} \exp\left(-\frac{1}{2}\left(\frac{u- \bar{\mu}}{\bar{\nu}}\right)^2\right) u\mathrm{d}u.
	\end{align*}
	Note that
	\begin{align*}
	    \int_{ \bar{\mu}}^{\infty} &\exp\left(-\frac{1}{2}\left(\frac{u- \bar{\mu}}{\bar{\nu}}\right)^2\right) u \mathrm{d}u
	    \\=&
	    \int_{0}^{\infty} \exp\left(-\frac{1}{2}\left(\frac{z}{\bar{\nu}}\right)^2\right) (z+ \bar{\mu})\mathrm{d}z\\
	    =&
	    \int_{0}^{\infty} \exp\left(-\frac{z^2}{2\bar{\nu}^2}\right) z\mathrm{d}z+ \int_{0}^{\infty} \exp\left(-\frac{z^2}{2\bar{\nu}^2}\right)  \bar{\mu}\mathrm{d}z\\
	    = &\bar{\nu}^2+\sqrt{\pi/2}\bar{\nu}\bar{\mu}\\
	    = &\bar{\nu}(\bar{\nu}+\sqrt{\pi/2}\bar{\mu}).
	\end{align*}
	Therefore,
	\begin{align*}
	    \mathbb{E}\{\|\revise{v[k',k]}\|^2\}
	   \leq&  \bar{\mu}^2
	   + 2\bar{\gamma}\bar{\nu}(\bar{\nu}+\sqrt{\pi/2}\bar{\mu})\\
	   =&\frac{\lambda s^2d^2n^2}{\lfloor k/T_0 \rfloor \alpha +\lambda}\\&+\frac{((\lfloor k/T_0 \rfloor + 1) \gamma +\lambda)}{\lambda}\\
    &\quad\times\frac{rd^2n^{7/2}(
	   rn^{1/2}+
	   \sqrt{\pi\lambda}s )}{(\lfloor k/T_0 \rfloor \alpha +\lambda)}.
	\end{align*}
	This concludes the proof. 

\fi 

\section{Proof of Theorem~\ref{tho:safe_projection}} \label{proof:tho:safe_projection}
We first show that the projection of the control action $\bar{u}[k]$ to a safe set can be done by solving:
\begin{subequations} \label{eqn:optim_safe}
	\begin{align}
		u[k]\in&\argmin_{u\in\mathcal{U}}  \mathbf{d}(u,\bar{u}[k]),\\
		&\quad\, \mathrm{s.t.}  \quad H[k+1](\hat{A}[k]x[k] + \hat{B}[k] u +v +w)\nonumber\\
        &\hspace{2.1in}\leq h[k+1], \quad 
		\nonumber\\
		& \hspace{.5in}\forall w: w^\top w \leq \frac{2rn}{\delta},\nonumber\\
		& \hspace{.5in} \forall v: v^\top v \leq \frac{n^2d^2}{\sigma_{\min}(V[k])}\beta_k^2\left(\frac{\delta}{2n}\right),
	\end{align}
\end{subequations}
Note that $
x[k+1]=
\hat{A}x[k]+\hat{B}u[k]
+
v[k]+w[k]$, 
where $v[k]=(A-\hat{A}[k])x[k]+(B-\hat{B}[k])u[k].$
Therefore, proving the safety of the projected control action in~\eqref{eqn:optim_safe} follows from bounding the noise and perturbation terms $v[k]$ and $w[k]$ with high probability. Proposition~\ref{cor:confidence} implies that
\begin{align*}
	\mathbb{P}\left\{\|v[k]\|^2
	\leq \zeta^2 n^2\beta_k^2\left(\!\frac{\delta}{2n}\!\right)\!\right\}
    \!=&\mathbb{P}\left\{\|v[k]\|
	\leq \zeta n\beta_k\left(\!\frac{\delta}{2n}\!\right)\!\right\}\\
    \geq& 1-\frac{\delta}{2},
\end{align*}
where $\zeta=d/{\sqrt{\sigma_{\min}(V[k])}}$. For the process noise, we have
\begin{align*}
	\mathbb{P}\{w[k]^\top (rI)^{-1} w[k]\leq \varepsilon\}
	\geq &
	\mathbb{P}\{w[k]^\top W^{-1} w[k]\leq \varepsilon\}\\
	\geq & 1-\frac{\mathbb{E}\{w[k]^\top W^{-1} w[k]\}}{\varepsilon} \\
	= & 1-\frac{n}{\varepsilon},
\end{align*}
where the first inequality follows from Assumption~\ref{assum:bound_on_variance} and the second inequality follows from an application of Markov's inequality for scalar random variables~\cite[\S\,2.1]{boucheron2013concentration}. 
Selecting $\varepsilon=(2n)/\delta$ gives
$\mathbb{P}\{w[k]^\top (rI)^{-1} w[k]\leq (2n)/\delta\}\geq 1-\delta/2.$ 
Finally, we note that
\begin{align*}
	&\mathbb{P}\left\{w[k]^\top w[k]\leq \frac{2rn}{\delta} \bigwedge \|v[k]\|
	\leq \zeta n\beta_k\left(\frac{\delta}{2n}\right)\right\}
	\\ &\hspace{.1in}=1\!-\!\mathbb{P}\left\{\!w[k]^\top w[k]> \frac{2rn}{\delta} \bigvee \|v[k]\|
	\!>\! \zeta n\beta_k\left(\!\frac{\delta}{2n}\!\right)\!\right\}\\
	&\hspace{.1in}\geq 1\!-\!\mathbb{P}\left\{w[k]^\top w[k]> \frac{2rn}{\delta}\right\} \\
    &\hspace{.32in}-\!\mathbb{P}\left\{\|v[k]\|
	> \zeta n\beta_k\left(\frac{\delta}{2n}\right)\right\}\\
	&\hspace{.1in}=1\!-\!\delta,
\end{align*}
where the inequality follows from the union bound. 

Finally, Lemma~\ref{lemma:robust} can be used to eliminate $v$ in~\eqref{eqn:optim_safe} to obtain
	\begin{align*}
		u[k]\in&\argmin_{u\in\mathcal{U}}  \mathbf{d}(u,\bar{u}[k]),\\
		&\quad\, \mathrm{s.t.}  \quad H[k\!+\!1](\hat{A}[k]x[k] + \hat{B}[k] u +w)\leq h[k\!+\!1]\\
        &\hspace{1in}-e[k\!+\!1], \quad \forall w: w^\top w \leq \frac{2rn}{\delta},
	\end{align*}
	where $e_i[k+1]=({dn}/\sqrt{\sigma_{\min}(V[k])}) \beta_k({\delta}/({2n}))\|H_i[k+1]^\top\|.$ An additional application of Lemma~\ref{lemma:robust} to eliminate $w$ concludes the proof.

\ifdefined\LONGVERSION
 \section{Proof of Proposition~\ref{prop:PoE_model}}
 \label{proof:prop:PoE_model}
 Let us define
\begin{align*}
	\Xi_{T_0}(k):=\begin{bmatrix} 
		\displaystyle\sum_{t=k}^{k+T_0-1} x[t]x[t]^\top &
		\displaystyle\sum_{t=k}^{k+T_0-1} x[t]u[t]^\top\\[.5em]
		\displaystyle\sum_{t=k}^{k+T_0-1} u[t]x[t]^\top &
		\displaystyle\sum_{t=k}^{k+T_0-1} u[t]u[t]^\top 
	\end{bmatrix}.
\end{align*}
Note that 
\begin{align*}
	V[k]=\lambda I&\!+\! \Xi_{T_0}(0)\!+\!\Xi_{T_0}(T_0)\!+\!\cdots\!+\!\Xi_{T_0}((\lfloor k/T_0 \rfloor\!-\!1) T_0)\\
	&+\begin{bmatrix} 
		\displaystyle\sum_{t=\lfloor k/T_0 \rfloor T_0}^{k} x[t]x[t]^\top &
		\displaystyle\sum_{t=\lfloor k/T_0 \rfloor T_0}^{k} x[t]u[t]^\top\\[.5em]
		\displaystyle\sum_{t=\lfloor k/T_0 \rfloor T_0}^{k} u[t]x[t]^\top &
		\displaystyle\sum_{t=\lfloor k/T_0 \rfloor T_0}^{k} u[t]u[t]^\top 
	\end{bmatrix},
\end{align*}
and as a result
\begin{align*}
	\lambda I+ \sum_{j=0}^{\lfloor k/T_0 \rfloor -1}\Xi_{T_0}(jT_0) \preceq V[k] \preceq \lambda I+\sum_{j=0}^{\lfloor k/T_0 \rfloor }\Xi_{T_0}(jT_0).
\end{align*}
Under persistence of excitation, $\alpha I \preceq \Xi_{T_0}(k) \preceq \gamma I$ for all $k$ and therefore 
\begin{align} \label{eqn:matrix_inequality}
	(\lfloor k/T_0 \rfloor \alpha +\lambda) I \preceq V[k] \preceq ((\lfloor k/T_0 \rfloor + 1) \gamma +\lambda) I.
\end{align}
Hence, it must be that $\det(V[k])=\mathcal{O}(k^n)$ and $\sigma_{\min}(V[k])=\mathcal{O}(k)$. Recalling the definitions of $\beta_k$ and $\zeta$, we obtain 
\begin{align*}
	\zeta n\beta_k\left(\frac{\delta}{2n}\right)
	&=\frac{\displaystyle nd\left[r\sqrt{2\log\left(\frac{2n\det(V[k])^{1/2}}{\lambda^{n/2}\delta}\right)}\!+\!\lambda^{1/2}s\right]}{\sqrt{\sigma_{\min}(V[k])}}
	\\&=\mathcal{O}\left(\frac{\sqrt{\log(\det(V[k])^{1/2})}}{\sqrt{\sigma_{\min}(V[k])}}\right)
	\\&=\mathcal{O}\left(\frac{\sqrt{\log(k)}}{\sqrt{k}}\right), 
\end{align*}
which concludes the proof as $\lim_{k\rightarrow \infty}\sqrt{\log(k)}/\sqrt{k}=0$.

\section{Proof of Theorem~\ref{tho:safe_with_W}}
\label{proof:tho:safe_with_W}

First, we must prove that safety can be achieved by projecting the control action $\bar{u}[k]$ using
\begin{subequations} \label{eqn:optim_safe_also_W}
	\begin{align}
		u[k]\!\in\! &\argmin_{u\in\mathcal{U}}  \mathbf{d}(u,\bar{u}[k]),\\
		&\quad\, \mathrm{s.t.}  \, H[k\!+\!1](\hat{A}[k_0]x[k]\!+\!\hat{B}[k_0]u\!+\!v\!+\!w)\!\leq\! h[k\!+\!1], \quad 
		\nonumber\\
		& \hspace{.5in}\forall w: w[k]^\top \Pi_{k,k_0} w[k]\leq \frac{3}{n\delta},\nonumber\\
		& \hspace{.5in} \forall v: v^\top v \leq \frac{n^2d^2}{\sigma_{\min}(V[k])}\beta_k^2\left(\frac{\delta}{3n}\right),
	\end{align}
\end{subequations}
where 
\begin{align*}
	\Pi_{k,k_0}^{-1}:=&\revise{\widehat{W}[k-1,k_0]}\\
 &+\sqrt{\frac{3}{\delta}\left(\!2L_4(k_0)^2\!+\!\frac{8rL_2(k_0)}{k-k_0}\!+\!\frac{2r^2n(n\!+\!1)}{k-k_0}\right)} I,
\end{align*}
To do so, define 
$$
    \revise{\overline{W}[k,k_0]}=\frac{1}{k-k_0}\sum_{t=k_0+1}^k w[k]w[k]^\top,
$$
and
$
    \revise{v[k,k_0]}=(A-\hat{A}[k_0])x[k]+(B-\hat{B}[k_0])u[k].
$
Because $\hat{w}[k]=w[k]+\revise{v[k,k_0]}$, we get
\begin{align*}
    \mathbb{E}\{&\trace((\revise{\widehat{W}[k,k_0]}-\revise{\overline{W}[k,k_0]})^2)\}\\
    &\leq \mathbb{E}\{(\trace(\revise{\widehat{W}[k,k_0]}-\revise{\overline{W}[k,k_0]}))^2\}\\
    &=\mathbb{E}\left\{\left(\frac{1}{k-k_0}\left(\sum_{t=k_0+1}^k \|v[t]\|^2+2w[t]^\top v[t]\right)\right)^2\right\}\\
    &=\frac{1}{(k-k_0)^2}\sum_{t=k_0+1}^k\sum_{t'=k_0+1}^k \mathbb{E}\{\|v[t]\|^2\|v[t']\|^2\}
    \\&\quad+\frac{4}{(k-k_0)^2}\sum_{t=k_0+1}^k\sum_{t'=k_0+1}^k \mathbb{E}\{\|v[t]\|^2w[t']^\top v[t']\}
    \\&\quad+\frac{4}{(k-k_0)^2}\sum_{t=k_0+1}^k\sum_{t'=k_0+1}^k \mathbb{E}\{v[t]^\top  w[t] w[t']^\top v[t']\}\\
    &=\frac{1}{(k-k_0)^2}\sum_{t=k_0+1}^k\sum_{t'=k_0+1}^k \mathbb{E}\{\|v[t]\|^2\|v[t']\|^2\}
    \\&\quad+\frac{4}{(k-k_0)^2}\sum_{t=k_0+1}^k \mathbb{E}\{v[t]^\top  w[t] w[t]^\top v[t]\}\\
    &\leq L_4(k_0)+\frac{4rL_2(k_0)}{k-k_0},
\end{align*}
where the first inequality is a consequence of~\cite[\S4.1.2, Item~(13)]{l1996handbook}, the penultimate equality follows from that $w[k']$ and $v[k]$ are independent for $k',k> k_0$, and the last inequality follows from Proposition~\ref{prop:upperbound_expectation}. Furthermore, following the same line of reasoning as in the proof of Lemma III.7 in~\cite{farokhi_safe_lerning}, we get 
    \begin{align*}
        \mathbb{E}\{\trace((\revise{\overline{W}[k,k_0]}-W)^2)\}
        &=\frac{\trace(W)^2+\trace(W^2)}{k-k_0}\\
        &\leq \frac{r^2n(n+1)}{k-k_0}.
    \end{align*}
    We also have
\begin{align*}
    \mathbb{E}\{&\trace((\revise{\widehat{W}[k,k_0]}-W)^2)\}\\
    &=\mathbb{E}\{\trace((\revise{\widehat{W}[k,k_0]}-W+\revise{\overline{W}[k,k_0]} -\revise{\overline{W}[k,k_0]})^2)\}\\
    &\leq 2\mathbb{E}\{\trace((\revise{\widehat{W}[k,k_0]} -\revise{\overline{W}[k,k_0]})^2)\}\\
    &\quad+2\mathbb{E}\{\trace((\revise{\overline{W}[k,k_0]}-W)^2)\}\\
    &\leq 2L_4(k_0)^2+\frac{8rL_2(k_0)}{k-k_0}+\frac{2r^2n(n+1)}{k-k_0}. 
\end{align*}
The matrix version of Chebyshev's inequality~\cite{farokhi_safe_lerning} results in
	\begin{align*}
		\mathbb{P}\{&W-\revise{\widehat{W}[k,k_0]}\preceq \varepsilon I\}\\
		&\geq 1-\trace(\mathbb{E}\{ (\revise{\widehat{W}[k,k_0]}-W)^{2}\varepsilon^{-2}\})\\
		&\geq  1-\frac{1}{\varepsilon^2}\underbrace{\left(2L_4(k_0)^2+\frac{8rL_2(k_0)}{k-k_0}+\frac{2r^2n(n+1)}{k-k_0}\right)}_{:=\ell(k_0,k)}.
	\end{align*}
	Setting $ \varepsilon^2 \!=\! 3\ell(k_0,k)/\delta$ shows that
	\begin{align*}
		\mathbb{P}\left\{W-\revise{\widehat{W}[k,k_0]}\preceq \sqrt{\frac{3\ell(k_0,k)}{\delta}} I\right\}
		&\geq 1-\frac{\delta}{3}.
	\end{align*}
    Define
    \begin{align*}
        \Omega:=\revise{\widehat{W}[k-1,k_0]}+\sqrt{\frac{3\ell(k_0,k-1)}{\delta}} I.
    \end{align*}
	Hence,
    \begin{align*}
    \mathbb{P}&\left\{w[k]^\top \Omega^{-1}w[k]\leq \rho \right\}
    \\=&\mathbb{P}\left\{\!w[k]^\top\Omega^{-1}w[k]\leq \rho \Big| W\!-\!\revise{\widehat{W}[k,k_0]}\!\preceq \!\sqrt{\frac{3\ell(k_0,k)}{\delta}} I\right\}\\
    &\quad\times\mathbb{P}\!\left\{W-\revise{\widehat{W}[k,k_0]}\preceq \sqrt{\frac{3\ell(k_0,k)}{\delta}} I\!\right\}\\
    &\!+\!\mathbb{P}\!\left\{\!w[k]^\top\Omega^{-1}w[k]\!\leq\! \rho \Big| W\!-\!\revise{\widehat{W}[k,k_0]}\!\npreceq \!\sqrt{\frac{3\ell(k_0,k)}{\delta}} I\!\right\}\\
    &\quad\times\mathbb{P}\left\{W-\revise{\widehat{W}[k,k_0]}\npreceq \sqrt{\frac{3\ell(k_0,k)}{\delta}} I\right\}\\
    \geq &\mathbb{P}\left\{\!w[k]^\top\Omega^{-1}w[k]\leq \rho \Big| W\!-\!\revise{\widehat{W}[k,k_0]}\!\preceq \!\sqrt{\frac{3\ell(k_0,k)}{\delta}} I\right\}\\
    &\quad\times\mathbb{P}\left\{W-\revise{\widehat{W}[k,k_0]}\preceq \sqrt{\frac{3\ell(k_0,k)}{\delta}} I\right\}\\
    \geq &\mathbb{P}\left\{w[k]^\top W^{-1}w[k]\leq \rho \right\}\left(1-\frac{\delta}{3}\right)\\
    \geq &\left(1-\frac{\mathbb{E}\{w[k]^\top W^{-1}w[k]\}}{\rho}\right)\left(1-\frac{\delta}{3}\right)\\
    = &\left(1-\frac{n}{\rho}\right)\left(1-\frac{\delta}{3}\right).
	\end{align*}
	Selecting $\rho=3/(n\delta)$, we get
	\begin{align*}
    \mathbb{P}\left\{w[k]^\top\Omega^{-1}w[k]\leq \rho \right\}
    &\geq \left(1-\frac{\delta}{3}\right)^2\\
    &\geq \left(1-\frac{2\delta}{3}\right).
	\end{align*}
	Also, Corollary~\ref{cor:confidence} implies that
	\begin{align*}
		\mathbb{P}\Bigg\{\|v[k]\|^2
		\leq& \zeta^2 n^2\beta_k^2\left(\frac{\delta}{3n}\right)\Bigg\}\\
        &=\mathbb{P}\left\{\|v[k]\|
		\leq \zeta n\beta_k\left(\frac{\delta}{3n}\right)\right\}\\&\geq 1-\frac{\delta}{3},
	\end{align*}
	where $\zeta=d/\sqrt{\sigma_{\min}(V[k])}$. 
	Finally, we note that
	\begin{align*}
		\mathbb{P}&\left\{w[k]^\top\Omega^{-1}w[k]\leq \frac{3}{n\delta} \bigwedge \|v[k]\|
		\leq \zeta n\beta_k\left(\frac{\delta}{3n}\right)\right\}
	   \\ &=1\!-\!\mathbb{P}\left\{\!w[k]^\top\Omega^{-1}w[k]\!>\! \frac{3}{n\delta} \bigvee \|v[k]\|
		\!>\! \zeta n\beta_k\!\left(\frac{\delta}{3n}\right)\!\right\}\\
		&\geq 1\!-\!\mathbb{P}\left\{w[k]^\top\Omega^{-1}w[k]> \frac{3}{n\delta} \right\} \\
        &\hspace{.2in}-\mathbb{P}\left\{\|v[k]\|
		> \zeta n\beta_k\left(\frac{\delta}{3n}\right)\right\}\\
		&=1\!-\!\delta,
	\end{align*}
	where the inequality is the consequence of the union bound. 

    Finally, we can rewrite~\eqref{eqn:optim_safe_also_W} as~\eqref{eqn:optim_safe_tightened_with_W} using Lemma~\ref{lemma:robust}.
\fi 

\end{document}